\documentclass[11pt]{article}

\usepackage[margin=1in]{geometry}
\usepackage{amsmath,amssymb,amsthm}
\usepackage{graphicx}
\usepackage{booktabs}
\usepackage{algorithm}
\usepackage{algorithmic}
\usepackage{booktabs}
\usepackage{multirow}
\newtheorem{theorem}{Theorem}
\newtheorem{lemma}[theorem]{Lemma}
\newtheorem{Proposition}{Proposition}

\usepackage{pgfplots}
\pgfplotsset{compat=1.18}
\usepackage{amsmath}
\newtheorem{definition}{Definition}
\usepackage{hyperref}
\usepackage{natbib}  
\usepackage{caption} 
\usepackage{booktabs}
\usepackage{booktabs}
\usepackage{pgfplots}

\usepackage{pgfplots}
\pgfplotsset{compat=1.18}
\usepackage{newfloat}
\usepackage[T1]{fontenc}
\usepackage{listings}
\title{Individual Fairness in Hierarchical Clustering}
\author{
 Binita Maity and Shrutimoy Das\\
Indian Institute of Technology, Gandhinagar\\
\texttt{\{binitamaity,shrutimoydas\}@iitgn.ac.in}
}

\begin{document}

\maketitle

\begin{abstract}
Hierarchical clustering produces ultrametric representations that impose strong global geometric constraints and may distort local similarities in ways that disproportionately affect individual data points. We study hierarchical clustering under an individual fairness requirement that bounds relative distortion within local $k$-nearest neighborhoods. We formulate this requirement as a feasibility problem over dominated ultrametrics and characterize the minimal multiplicative slack required for feasibility. We identify a sharp local threshold, prove stability under bounded perturbations, establish monotonicity in $k$, and show an intrinsic $\Theta(\log n)$ separation between local and global realizability. Experiments on synthetic and real world datasets support our theoretical results.
\end{abstract}

\section{Introduction}

Hierarchical clustering is a classical problem in unsupervised learning \cite{johnson1967hierarchical}, where pairwise distances are given and the goal is to construct a nested sequence of partitions represented as a \emph{dendrogram.} Unlike flat clustering, hierarchical clustering produces a multi resolution representation of the data, with merge heights defining an \emph{ultrametric} over the points \cite{murtagh1983survey}. Ultrametrics impose strong global geometric constraints, i.e.,  every triple must satisfy the ultrametric inequality, meaning the two largest pairwise distances coincide. Hierarchical clustering is widely used in exploratory data analysis, computational biology, information retrieval, and social network analysis.

Most hierarchical clustering methods focus solely on optimizing a global objective or following a greedy linkage rule. However, such approaches provide no guarantees about how individual data points are treated locally. In particular, merge decisions may distort distances between nearby points in ways that disproportionately affect certain individuals. Because hierarchical structure propagates upward in the tree, early distortions influence all coarser resolutions. Thus, even if the overall hierarchy appears reasonable, individual points may experience large relative distortion compared to their local neighborhoods.

This limitation motivates our problem. We investigate hierarchical clustering under an \emph{individual fairness} requirement, i.e., similar individuals should be treated similarly \cite{dwork2011fairnessawareness}. Rather than optimizing a global objective alone, we ask whether there exists an ultrametric representation that respects local similarity constraints up to a bounded multiplicative slack. Our goal is to characterize the minimal slack required to reconcile local Lipschitz type constraints with global ultrametric structure.

We formalize this problem as a feasibility question over \emph{dominated ultrametrics}. Given a metric space $(V,d)$, we seek an ultrametric $u$ that (i) dominates $d$ (so that no pair is contracted), and (ii) satisfies multiplicative fairness constraints within $k$-nearest neighborhoods. The key quantity in our analysis is the minimal multiplicative slack $\alpha^\star(d,k)$ for which such an ultrametric exists.

Our main theoretical results, presented in Section~\ref{sec4}, characterize the feasibility landscape of individually fair hierarchical clustering. In Theorem~\ref{thm:local-threshold}, we identify a sharp local multiplicative threshold $\alpha^{\mathrm{mut}}_k(d)$ determined by scale heterogeneity within mutual $k$-nearest neighborhoods and prove that slack below this value renders fairness infeasible. In Theorem~\ref{thm:stability}, we establish Lipschitz stability of this threshold under bounded $\ell_\infty$ perturbations that preserve neighborhood identities. Finally, Theorem~\ref{th:mono} shows that the minimal feasible slack $\alpha^\star(d,k)$ is nondecreasing in the neighborhood parameter $k$, providing structural control over how strengthening local fairness constraints affects global feasibility.

Crucially, we establish that local conditions do not fully determine global realizability. In Proposition~\ref{prop:intrinsic-gap}, we construct metric families for which 
$\alpha^{\mathrm{mut}}_k(d)=1$, yet every dominating ultrametric incurs 
$\Theta(\log n)$ multiplicative distortion. 
This demonstrates an intrinsic separation between local similarity structure 
and global ultrametric geometry. 
Complementing this lower bound, Theorem~\ref{thm:sufficiency} shows that classical 
tree embedding techniques yield matching $O(\log n)$ sufficiency guarantees. Section~ \ref{sec5} presents a fairness constrained agglomerative clustering (FCAC) algorithm \ref{alg:fcac} that enforces dominance and local Lipschitz constraints during merging.

In section \ref{experi}, we empirically evaluate our framework on synthetic and real world datasets, observing sharp feasibility thresholds, saturation in neighborhood size, and dataset dependent distortion regimes consistent with the theory.

In summary, our contributions are as follows:
\begin{itemize}
    \item We formulate individually fair hierarchical clustering as a feasibility problem over dominated ultrametrics.
    \item We identify a sharp local feasibility threshold $\alpha^{\mathrm{mut}}_k(d)$ and prove that slack below this value renders fairness infeasible.
    \item We establish stability under bounded perturbations and prove that the minimal feasible slack $\alpha^\star(d,k)$ is monotone in $k$.
    \item We prove an intrinsic local global separation, showing that some metric families require $\Theta(\log n)$ distortion despite trivial local structure.
    \item We empirically validate the theoretical claims on five different synthetic and real world datasets.
\end{itemize}

\section{Related Work}

\paragraph{Individual fairness and clustering.}
Individual fairness requires that similar individuals be treated similarly, formalized via Lipschitz constraints with respect to a task specific metric~\cite{dwork2011fairnessawareness}. In recent years, few algorithms for individually fair clustering have been proposed and its scalable variants ~\cite{mahabadi2020individual,han2023approx,amagata,bateni2024scalablealgorithmindividuallyfair,binita,maity2025localsearchbasedindividuallyfair}. In contrast, we study the structural feasibility of individual fairness in hierarchical clustering.

\paragraph{Hierarchical clustering and tree embeddings.}
Hierarchical clustering and its ultrametric interpretation are classical~\cite{johnson1967hierarchical,murtagh2017algorithms,gan2020data}. The connection between dendrograms and tree metrics relates the problem to combinatorial and geometric structures~\cite{serre2002trees,papadimitriou1998combinatorial}. Classical probabilistic tree embedding results~\cite{bartal1996probabilistic,fakcharoenphol2004approximating,kannan1995tree,abraham2012using} show that any metric admits $O(\log n)$ distortion embeddings into trees; these bounds underpin our sufficiency guarantees and interpret fairness slack as embedding distortion. Objective based analyses such as the Dasgupta cost and its approximations~\cite{dasgupta2016cost,cohen2019hierarchical,moseley2023approximation} focus on global optimization rather than feasibility under local fairness constraints.

\paragraph{Constrained hierarchical clustering.}
Prior work studies hierarchical clustering under pairwise or triplet constraints and analyzes the behavior of greedy linkage rules~\cite{bilenko2004integrating,davidson2005agglomerative}. In contrast, we provide a geometric feasibility perspective on individual fairness, characterizing intrinsic distortion bounds imposed by ultrametric structure rather than optimizing under fixed constraints.

\section{Problem Setup}
In this section, we define all the notation we use throughout this paper. Let $V=\{1,\dots,n\}$ be a finite set equipped with a metric $d$.

\begin{definition}[Dendrogram \cite{gan2020data, murtagh1983survey}]
A dendrogram on $V$ is a rooted tree $T$ whose leaves are in bijection with $V$, together with a height function
$h:\mathrm{nodes}(T)\to\mathbb{R}_{\ge 0}$
such that:
(i) $h(\ell)=0$ for every leaf $\ell$, and
(ii) if $v$ is an ancestor of $u$ then $h(v)\ge h(u)$.
\end{definition}

For $i,j\in V$, let $\mathrm{LCA}(i,j)$ denote their lowest common ancestor.
The induced merge height is
\[
u(i,j):=h(\mathrm{LCA}(i,j)).
\]

\begin{definition}[Ultrametric \cite{gan2020data}]
A function $u:V\times V\to\mathbb{R}_{\ge0}$ is an ultrametric if
\begin{enumerate}
\item $u(i,i)=0$,
\item $u(i,j)=u(j,i)$,
\item $u(i,j)\le \max\{u(i,k),u(k,j)\}$ for all $i,j,k\in V$.
\end{enumerate}
\end{definition}

ultrametrics are in one-to-one correspondence with dendrograms (up to tree isomorphism). Let $\mathcal{U}$ denote the set of ultrametrics on $V$. For $k \ge 1$, let $N_k(i)$ denote the set of $k$ nearest neighbors of $i$
under $d$, excluding $i$ itself. Ties are broken using a fixed
arbitrary total order on $V$ to ensure determinism. For stability results, we assume strict separation at the $k$-th
distance boundary, i.e.,
\[
d(i,j_k) < d(i,j_{k+1}) \quad \forall i,
\]
where $j_k$ denotes the $k$-th \emph{nearest neighbor} of $i$.


\begin{definition}[Domination]
An ultrametric $u$ \emph{dominates} $d$ if
\[
u(i,j) \ge d(i,j) \quad \forall i,j\in V.
\]
\end{definition}

We adopt the notion of individual fairness introduced in \cite{dwork2011fairnessawareness}, and specialize it to ultrametrics.
\begin{definition}[Individually Fair Ultrametric]
Fix $k\ge1$, $\alpha\ge1$, and $\beta\ge0$.
Define the feasible set
\small
\[
\mathcal U_{IF}^{+}
=
\left\{
u \in \mathcal U :
\begin{cases}
u(i,j) \ge d(i,j) & \forall i,j, \\[4pt]
u(i,j) \le \alpha d(i,j)+\beta
& \forall i\in V,\ j\in N_k(i)
\end{cases}
\right\}.\]
\end{definition}

The first condition avoids collapsing distances, and the second enforces local Lipschitz fairness among neighbors. Symmetry $u $ ensures the constraint holds in both directions.

\begin{definition}[Feasibility]
An instance $(V,d,k,\alpha,\beta)$ is \emph{feasible}
if $\mathcal U_{IF}^{+}\neq\emptyset$.
\end{definition}

We define a \emph{minimal multiplicative slack} parameter measuring the smallest distortion needed for feasibility.

\begin{definition}[Minimal Slack]
For fixed $(V,d,k,\beta)$, define the minimal multiplicative slack
\[
\alpha^\star
:=
\inf\{\alpha \ge 1 : \mathcal U_{IF}^{+}\neq\emptyset\}.
\]
\end{definition}

For two metrics $d,d'$ on $V$, define
\[
\|d-d'\|_\infty
=
\max_{i,j} |d(i,j)-d'(i,j)|.
\]

We say $d'$ preserves $k$-nearest neighborhoods of $d$
if $N_k^{d'}(i)=N_k^{d}(i)$ for all $i$.

Given $u\in\mathcal U$, define the height weighted dissimilarity objective
\[
C(u)=\sum_{i<j} d(i,j)\,u(i,j).
\]

The individually fair hierarchical clustering problem is
\[
\min_{u\in\mathcal U_{IF}^{+}} C(u).
\]

\section{Feasibility Theory}
\label{sec4}
We analyze when a feasible ultrametric exists. For $k \ge 1$ and parameters $(\alpha,\beta)$, we ask:

For which $(\alpha,\beta)$ does there exist 
$u \in \mathcal U$ such that
\[d(i,j) \le u(i,j) \quad \forall i,j,
\]
and
\[
u(i,j) \le \alpha d(i,j) + \beta
\]
for all required neighbor pairs?

Throughout this section we first consider the multiplicative case $\beta=0$.

\paragraph{Mutual fairness.}
\label{mutualfair}
In the mutual formulation, the fairness constraint is imposed only on pairs satisfying
\[
j \in N_k(i)
\quad \text{and} \quad
i \in N_k(j).
\]
The necessary and sufficient conditions below are stated for this mutual version.

\begin{definition}[Local Mutual Heterogeneity Ratio]
Define
\[
\alpha^{\mathrm{mut}}_k(d)
=
\max_{i \in V}
\max_{\substack{
j,\ell \in N_k(i)\\
i \in N_k(j),\, i \in N_k(\ell)\\
d(i,j) < d(i,\ell)
}}
\frac{d(i,\ell)}{d(i,j)}.
\]
\end{definition}

This quantity measures the maximal scale separation inside any mutual
$k$ nearest neighbor star.

If $d$ is locally ultrametric at scale $k$,
then $\alpha^{\mathrm{mut}}_k(d)=1$.

\subsection{Necessary Slack Under Mutual Neighborhoods}
We now show that the local heterogeneity ratio provides a necessary lower bound on any feasible multiplicative slack. 
\begin{theorem}[Local Necessary Slack]
\label{thm:local-threshold}
Let $u \in \mathcal U$ satisfy:

\begin{enumerate}
\item Dominance: $d(x,y) \le u(x,y) \quad \forall x,y,$
\item Mutual fairness:
$u(x,y) \le \alpha d(x,y) \quad \text{whenever } x \in N_k(y) \text{ and } y \in N_k(x)$.

\end{enumerate}

Then necessarily $\alpha \ge \alpha^{\mathrm{mut}}_k(d).$
\end{theorem}
\begin{proof}
Fix $i$ and mutual neighbors $j,\ell$ attaining 
$\alpha^{\mathrm{mut}}_k(d)$ with
$d(i,j) < d(i,\ell)$.
Assume for contradiction that 
\[
\alpha < \frac{d(i,\ell)}{d(i,j)}.
\]

By mutual fairness, $u(i,j) \le \alpha d(i,j),$ and by dominance, $u(i,\ell) \ge d(i,\ell).$  Thus $u(i,j) < u(i,\ell)$.

By the ultrametric inequality,
\[
u(j,\ell) \le \max\{u(j,i),u(i,\ell)\} = u(i,\ell).
\]
On the other hand,
\[
u(i,\ell) \le \max\{u(i,j),u(j,\ell)\}.
\]
Since $u(i,j) < u(i,\ell)$, it follows that 
$u(j,\ell) \ge u(i,\ell)$.
Hence $u(j,\ell)=u(i,\ell)$.

Therefore
\[
d(i,\ell) \le u(i,\ell) = u(j,\ell) \le \alpha d(i,j),
\]
contradicting the assumption.
\end{proof}

The theorem shows that local geometric imbalance among mutual neighbors forces a proportional amount of multiplicative slack.

\subsection{Stability of the Minimal Slack}

\begin{definition}[$\varepsilon$-Stable Neighborhoods]
We say that $d$ has $\varepsilon$-stable $k$-neighborhoods if
\[
d(i,j_k) + 2\varepsilon < d(i,j_{k+1})
\quad \forall i,
\]
where $j_k$ and $j_{k+1}$ denote the $k$-th and $(k+1)$-th
nearest neighbors of $i$ under $d$.
\end{definition}
This condition guarantees that any perturbation of $d$ of size at most $\varepsilon$ in $\ell_\infty$ norm preserves the $k$-nearest neighbor sets.

Since $d$ has $\varepsilon$-stable $k$-neighborhoods and $\|d-d'\|_\infty \le \varepsilon$, we have $N_k^{d'}(i)=N_k^{d}(i)$ for all $i$. Consequently, the maximization defining $\alpha_k^{\mathrm{mut}}$ is over the same index set for $d$ and $d'$.
\begin{theorem}[Stability Bound]
\label{thm:stability}
Suppose $\|d-d'\|_\infty \le \varepsilon$ and suppose
$d$ has $\varepsilon$-stable $k$ neighborhoods.
Define
\[
\delta := \min_{i} \min_{j\in N_k(i)} d(i,j).
\]
If $\varepsilon < \delta/2$, then
\[
\left|
\alpha_k^{\mathrm{mut}}(d')
-
\alpha_k^{\mathrm{mut}}(d)
\right|
\le
\frac{2\varepsilon\big(\alpha_k^{\mathrm{mut}}(d)+1\big)}{\delta}.
\]
\end{theorem}

\begin{proof}

For $i \in V$, let $j_k$ and $j_{k+1}$ denote the $k$-th and $(k+1)$-th nearest
neighbors of $i$ under $d$.
Since, \[ \|d-d'\|_\infty \le \varepsilon, \]
\[d'(i,j_k) \le d(i,j_k)+\varepsilon, \qquad
d'(i,j_{k+1}) \ge d(i,j_{k+1})-\varepsilon.\]

By $\varepsilon$-stability, $d(i,j_k) + 2\varepsilon < d(i,j_{k+1}),$
hence $d'(i,j_k) < d'(i,j_{k+1}),$ so the $k$ nearest neighbor sets are identical for $d$ and $d'$, so \ref{mutualfair} preserved.

Fix a mutual triple $(i,j,\ell)$.
Define
\[
a = d(i,\ell), \qquad b = d(i,j),
\]
\[
a' = d'(i,\ell), \qquad b' = d'(i,j).
\]

Then $|a-a'| \le \varepsilon, \qquad |b-b'| \le \varepsilon.$

We compare the ratios, $$\frac{a'}{b'} - \frac{a}{b}
=
\frac{ab' - a'b}{bb'}
=
\frac{a(b'-b) + b(a-a')}{bb'}.$$

Taking absolute values, $$\left|
\frac{a'}{b'} - \frac{a}{b}
\right|
\le
\frac{a|b'-b| + b|a-a'|}{bb'}
\le
\frac{\varepsilon(a+b)}{bb'}.$$

By definition of $\alpha_k^{\mathrm{mut}}(d)$, $a \le \alpha_k^{\mathrm{mut}}(d)\, b.$ Hence $a+b \le (\alpha_k^{\mathrm{mut}}(d)+1)b.$

Thus $$\left|
\frac{a'}{b'} - \frac{a}{b}
\right|
\le
\frac{\varepsilon(\alpha_k^{\mathrm{mut}}(d)+1)b}{bb'}
=
\frac{\varepsilon(\alpha_k^{\mathrm{mut}}(d)+1)}{b'}.$$

By definition of $\delta,
b \ge \delta.$ Since $|b-b'|\le \varepsilon,$so, $$
b' \ge b-\varepsilon \ge \delta-\varepsilon.
$$

Therefore, $$\left|
\frac{a'}{b'} - \frac{a}{b}
\right|
\le
\frac{\varepsilon(\alpha_k^{\mathrm{mut}}(d)+1)}{\delta-\varepsilon}.$$

If $\varepsilon < \delta/2$, then $\delta-\varepsilon \ge \frac{\delta}{2}.$

Hence,
$$\left|
\frac{a'}{b'} - \frac{a}{b}
\right|
\le
\frac{2\varepsilon(\alpha_k^{\mathrm{mut}}(d)+1)}{\delta}.$$

The above bound holds for every mutual triple $(i,j,\ell)$; that is,
\[
\left|
\frac{d'(i,\ell)}{d'(i,j)}
-
\frac{d(i,\ell)}{d(i,j)}
\right|
\le C
\quad \text{for all mutual triples}.
\]
Since the mutual $k$-NN structure is preserved, the same collection of triples is used in the definition of both $\alpha_k^{\mathrm{mut}}(d)$ and $\alpha_k^{\mathrm{mut}}(d')$. Taking the maximum over all such triples yields
$$\left|
\alpha_k^{\mathrm{mut}}(d')
-
\alpha_k^{\mathrm{mut}}(d)
\right|
\le C.$$

\end{proof}

The bound shows that $\alpha_k^{\mathrm{mut}}$ varies Lipschitz-continuously with respect to $\|\cdot\|_\infty$ perturbations, provided neighborhoods remain stable.

\subsection{A Constructive Upper Bound}

The necessary condition above shows that local heterogeneity
imposes a lower bound on the required slack.
We now prove a complementary global upper bound;
for every metric, sufficiently large multiplicative slack
always guarantees feasibility.

\begin{theorem}[Constructive Sufficiency]
\label{thm:sufficiency}
For any finite metric $(V,d)$ on $n$ points and any $k \ge 1$,
there exists a dominated $\alpha$-fair ultrametric with
\[
\alpha = O(\log n).
\]
\end{theorem}

\begin{proof}[Proof sketch]
The result follows from classical tree-embedding theory.
By the theorem of ~\cite{fakcharoenphol2004approximating},
every $n$-point metric admits a dominating tree metric
with distortion $O(\log n)$.
Since ultrametrics are tree metrics, this yields an ultrametric $u$
satisfying
\[
d(i,j) \le u(i,j) \le O(\log n)\, d(i,j)
\quad \forall i,j.
\]
In particular, the multiplicative fairness constraint holds
for all required neighbor pairs.
Full details appear in Appendix~\ref{appendix:a}.
\end{proof}

While the $O(\log n)$ bound follows from classical tree-embedding theory, 
its role here is conceptually different. 
In standard embedding theory, distortion measures geometric approximation. 
In our setting, distortion becomes the minimal fairness slack required 
to reconcile local Lipschitz constraints with hierarchical structure. 
Thus, classical metric distortion directly quantifies the intrinsic 
cost of fairness in hierarchical representations.

\subsection{Intrinsic Gap Between Local and Global Feasibility.}

The local obstruction $\alpha_k^{\mathrm{mut}}(d)$ does not
fully characterize the minimal feasible global slack.
We show that there exist metrics with trivial local structure
yet requiring logarithmic global distortion.

\begin{Proposition}[Intrinsic Gap]
\label{prop:intrinsic-gap}
There exists a family of $n$-point metrics $(V_n,d_n)$
and a fixed constant $k$ such that
\[
\alpha_k^{\mathrm{mut}}(d_n)=1,
\qquad
\alpha^\star(d_n)=\Theta(\log n).
\]
\end{Proposition}

\begin{proof}[Proof sketch]
Let $G_n$ be a constant-degree expander graph and
let $d_n$ be its shortest-path metric.
For $k$ equal to the degree,
all mutual $k$-nearest neighbors lie at distance $1$,
so $\alpha_k^{\mathrm{mut}}(d_n)=1$.

However, it is classical that any embedding of an expander
into a tree metric incurs distortion $\Omega(\log n)$
\cite{linial1995geometry,bartal1996probabilistic}.
Since ultrametrics are tree metrics and must dominate $d_n$,
any feasible ultrametric must incur multiplicative slack
$\Omega(\log n)$.

The full construction and analysis appear in Appendix \ref{appendix:c}.
\end{proof}


We now establish a structural property of the minimal slack 
as a function of the neighborhood parameter $k$.

\begin{theorem}[Monotonicity in $k$]
\label{th:mono}
Let $d$ be a fixed metric on $V$. 
If $1 \le k_1 \le k_2$, then
\[
\alpha^\star(d,k_1) \le \alpha^\star(d,k_2).
\]
Equivalently, the minimal multiplicative slack required for 
feasibility is nondecreasing in the neighborhood size.
\end{theorem}

\begin{proof}
For each $k$, let $\mathcal U_{IF}^+(k,\alpha)$ denote the set of 
dominated ultrametrics satisfying the fairness constraints 
with neighborhood size $k$ and multiplicative slack $\alpha$.

Observe that if $k_1 \le k_2$, then for every $i \in V$,
\[
N_{k_1}(i) \subseteq N_{k_2}(i).
\]
Hence the fairness constraints imposed under $k_2$ include 
all constraints imposed under $k_1$, and possibly additional ones.

Therefore, for every $\alpha$,
\[
\mathcal U_{IF}^+(k_2,\alpha)
\subseteq
\mathcal U_{IF}^+(k_1,\alpha).
\]
In particular, if some $\alpha$ is feasible for $k_2$, 
it is also feasible for $k_1$. Taking the infimum over feasible $\alpha$ values yields
\[
\alpha^\star(d,k_1)
\le
\alpha^\star(d,k_2),
\]
establishing monotonicity.
\end{proof}

\section{Fairness Constrained Agglomerative Clustering (FCAC)}
\label{sec5}
The feasibility theory characterizes when individually fair
ultrametrics exist. We now present an algorithmic procedure
for constructing such ultrametrics when feasible.  FCAC, algorithm \ref{alg:fcac}, modifies classical agglomerative clustering by enforcing
fairness during merge selection while separating merge ordering
from merge height assignment to ensure dominance. Note that FCAC enforces one sided neighborhood fairness (for all $j \in N_k(i)$), whereas Theorem~1 is stated for the mutual variant. The one sided formulation imposes a superset of the local Lipschitz constraints, so the same local heterogeneity ratios provide necessary slack lower bounds.

Let $\ell(A,B)$ denote any linkage score
(single, complete, average, etc.), where linkage determines merge ordering only. Merge height is assigned as $h(A,B)
=
\max\!\left(
\ell(A,B),
\max_{i\in A,\, j\in B} d(i,j)
\right),$ ensuring $h(A,B) \ge \max_{i\in A,j\in B} d(i,j).$

At each step, \textsc{FindFairMerge}, algorithm \ref{alg:fair_merge}, 
returns the smallest linkage ordered merge satisfying the fairness constraint. If none exists, the algorithm returns \textsc{Infeasible}.  Accepted merges assign height $h(A,B)$ to all
$(i,j)\in A\times B$ and update clusters.

\begin{algorithm}[t]
\caption{\textsc{FCAC}}
\label{alg:fcac}
\begin{algorithmic}[1]
\REQUIRE Metric space $(V,d)$, neighborhoods $\{N_k(i)\}$,
fairness parameters $\alpha,\beta$, linkage function $\ell(\cdot,\cdot)$
\ENSURE Ultrametric $u$ satisfying dominance and fairness, or \textsc{Infeasible}

\STATE Initialize clusters $\mathcal{C} \gets \{\{i\} : i \in V\}$
\STATE Initialize $u(i,i) \gets 0$ for all $i \in V$

\WHILE{$|\mathcal{C}| > 1$}

    \STATE $(A^\ast,B^\ast,h^\ast) \gets$
           \textsc{FindFairMerge}
           $(\mathcal{C},\ell,d,\{N_k(i)\},\alpha,\beta)$

    \IF{$(A^\ast,B^\ast) = \textsc{Infeasible}$}
        \STATE \textbf{return} \textsc{Infeasible}
    \ENDIF

    \FOR{each $i \in A^\ast$}
        \FOR{each $j \in B^\ast$}
            \STATE $u(i,j) \gets h^\ast$
        \ENDFOR
    \ENDFOR

    \STATE $\mathcal{C} \gets
           (\mathcal{C} \setminus \{A^\ast,B^\ast\})
           \cup \{A^\ast \cup B^\ast\}$

\ENDWHILE

\STATE \textbf{return} $u$
\end{algorithmic}
\end{algorithm}

\begin{algorithm}[t]
\caption{\textsc{FindFairMerge}}
\label{alg:fair_merge}
\begin{algorithmic}[1]
\REQUIRE Current clusters $\mathcal C$, linkage $\ell(\cdot,\cdot)$,
metric $d$, neighborhoods $\{N_k(i)\}$, parameters $\alpha,\beta$
\ENSURE Fair merge pair $(A,B,h)$ or \textsc{Infeasible}

\STATE Sort all unordered pairs $(A,B)$ in increasing order of $\ell(A,B)$
\COMMENT{Linkage used only for ordering}

\FOR{each pair $(A,B)$ in sorted order}

    \STATE $h_{\min} \gets \max_{i\in A,\,j\in B} d(i,j)$
    \STATE $h \gets \max\big(\ell(A,B),\, h_{\min}\big)$
    \COMMENT{Dominance merge height}

    \STATE feasible $\gets$ TRUE

    \FOR{each $i \in A \cup B$}
        \FOR{each $j \in N_k(i)$}

            \IF{$i$ and $j$ are not yet in the same cluster}

                \IF{$h > \alpha d(i,j) + \beta$}
                    \STATE feasible $\gets$ FALSE
                    \STATE \textbf{break}
                \ENDIF

            \ENDIF

        \ENDFOR

        \IF{not feasible}
            \STATE \textbf{break}
        \ENDIF

    \ENDFOR

    \IF{feasible}
        \RETURN $(A,B,h)$
    \ENDIF

\ENDFOR

\STATE \textbf{return} \textsc{Infeasible}
\end{algorithmic}
\end{algorithm}

We first establish that the dominance merge heights are monotone, ensuring the resulting function is an ultrametric.

\begin{lemma}[Dominance Height Monotonicity]
Let $h_t$ denote the merge height at iteration $t$ of FCAC.
Then
\[
h_1 \le h_2 \le \dots \le h_{n-1}.
\]
\end{lemma}

\begin{proof}
At iteration $t$, suppose clusters $A_t,B_t$ are merged at height
\[
h_t
=
\max\!\Big(\ell(A_t,B_t),
\max_{i\in A_t,\,j\in B_t} d(i,j)\Big).
\]

By construction of \textsc{FindFairMerge}, all candidate pairs
$(A,B)$ are considered in nondecreasing order of linkage.
The algorithm selects the first feasible pair in this order.
Therefore, for every other candidate pair $(A,B)$
at iteration $t$, the associated dominance height $h(A,B)
=
\max\!\Big(\ell(A,B),
\max_{i\in A,\,j\in B} d(i,j)\Big)$
satisfies $h(A,B) \ge h_t,$ since otherwise $(A,B)$ would have been selected. Now consider any later iteration $s>t$.
Clusters at iteration $s$ are unions of clusters from iteration $t$,
and clusters are never split.
Let $(A_s,B_s)$ be the pair merged at iteration $s$,
with height $h_s$.

The dominance term $\max_{i\in A_s,\,j\in B_s} d(i,j)$ is computed over larger (or equal) sets than at iteration $t$.
Since taking a maximum over a superset cannot decrease its value, all cross cluster maxima are nondecreasing under cluster enlargement. Thus, the dominance term $(A_s,B_s)$ is at least as large as the dominance term for any corresponding merge at iteration $t$.

Because both the linkage order and the dominance term
cannot decrease over iterations, we must have $ h_s \ge h_t.$ Therefore, the merge heights are nondecreasing.
\end{proof}

In the next paragraph, we show that FCAC is sound; whenever it returns an ultrametric, the output satisfies both dominance and fairness.

\paragraph{Correctness} Each merge assigns $h=\max\{\ell(A,B),\max_{i\in A,j\in B} d(i,j)\},$ so $u(i,j)\ge d(i,j)$ for merged pairs. Since heights are nondecreasing, dominance holds globally. For $(i,j)$ with $j\in N_k(i)$, let $t$ be the iteration where $i,j$
first share a cluster. At that merge, the algorithm verifies $h_t\le\alpha d(i,j)+\beta,$ and $u(i,j)=h_t$.
Later merges do not alter heights, so fairness holds. If FCAC returns $u$, then $ u(i,j)\ge d(i,j), \quad u(i,j)\le\alpha d(i,j)+\beta
\quad \forall j\in N_k(i),$ hence $u\in U_{IF}^+$.

We next analyze the computational cost of \emph{FCAC} and discuss the complexity of the underlying feasibility problem.

\paragraph{Runtime of FCAC.}
At each iteration, FCAC examines all unordered pairs of current clusters.
In the worst case there are $O(n^2)$ candidate pairs.
For each candidate merge, the fairness check inspects at most
$O(nk)$ neighborhood constraints.
Since at most $n-1$ merges are performed,
the overall worst-case time complexity is
$O(n^3 k)$.
Space usage is $O(n^2)$ for storing the ultrametric.

\paragraph{Complexity of the feasibility problem.}
The underlying feasibility problem asks whether there exists
\emph{any} dominating ultrametric satisfying a system of
pairwise upper bound constraints.
For a fixed dendrogram, feasibility can be verified in polynomial time
by checking the induced inequalities.
However, deciding existence over the space of all ultrametrics
is combinatorial: merge heights are globally coupled through
the tree structure, and local consistency does not guarantee
global realizability (Section~\ref{sec4}).
Determining the precise computational complexity of this decision
problem, or identifying tractable structural subclasses,
remains an open direction.

\section{Experimental Results}
\label{experi}

We empirically evaluate $\alpha^\star(d,k)$ across synthetic and real world metrics to study (i) feasibility thresholds, (ii) the local global gap, and (iii) scaling behavior. In experiments, $\alpha^\star(d,k)$ denotes the smallest $\alpha$ for which FCAC returns a feasible ultrametric; hence reported values are algorithmic upper bounds on the true feasibility threshold.

\subsection{Synthetic Datasets}
We consider two complementary metric families:

\begin{itemize}
    \item \textbf{Gaussian mixtures in $\mathbb{R}^2$}, exhibiting
    heterogeneous but approximately hierarchical Euclidean geometry.
    \item \textbf{Shortest path metrics of random 3 regular graphs},
    isolating intrinsic global obstruction without local scale
    variation.
\end{itemize}

\paragraph{Feasibility Threshold in Euclidean Geometry}

We begin with Gaussian mixture instances consisting of two
well separated clusters in $\mathbb{R}^2$ with $k=3$. The local mutual neighborhood threshold equals $\alpha_k^{\mathrm{mut}}(d) = 5.0830$, while the minimal globally feasible slack is $\alpha^\star = 17.46.$ Feasibility is absent $\alpha \le 17$ and appears at
$\alpha \approx 18$, exhibiting a clear feasibility threshold
in the distortion parameter. Figure~\ref{fig:phase_transition}
illustrates this transition. The separation between $\alpha_k^{\mathrm{mut}}$
and $\alpha^\star$ demonstrates that local neighborhood
consistency substantially underestimates global ultrametric
distortion. Even in nearly hierarchical Euclidean geometry,
global feasibility requires significantly larger slack.

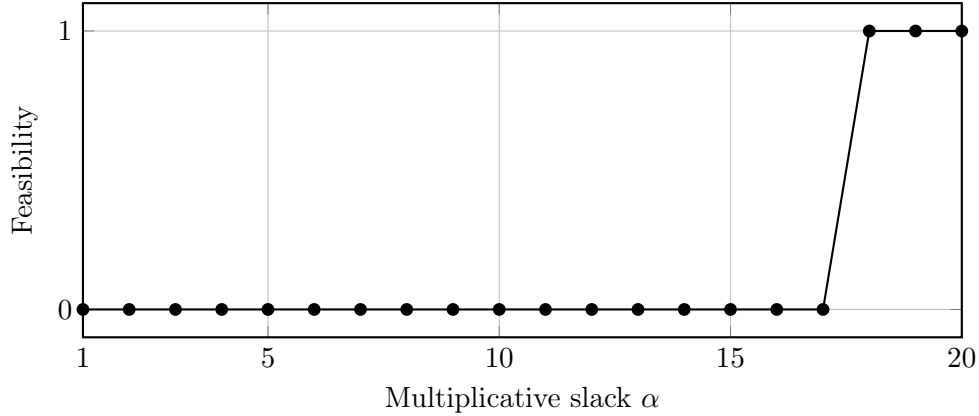
\begin{figure}[t]
\centering
\begin{tikzpicture}
\begin{axis}[
    width=0.8\linewidth,
    height=6cm,
    xlabel={Multiplicative slack $\alpha$},
    ylabel={Feasibility},
    ymin=-0.1, ymax=1.1,
    xmin=1, xmax=20,
    ytick={0,1},
    xtick={1,5,10,15,20},
    grid=major,
    thick
]
\addplot[
    color=black,
    mark=*,
] coordinates {
(1,0) (2,0) (3,0) (4,0) (5,0)
(6,0) (7,0) (8,0) (9,0) (10,0)
(11,0) (12,0) (13,0) (14,0)
(15,0) (16,0) (17,0)
(18,1) (19,1) (20,1)
};
\end{axis}
\end{tikzpicture}
\caption{Feasibility threshold in Gaussian mixture data.}
\label{fig:phase_transition}
\end{figure}

\paragraph{Empirical Evidence of the Local Global Gap}

To isolate purely global effects, we consider shortest path metrics of random $3$ regular graphs. 
These metrics exhibit uniform local scale: for $k=3$, all mutual neighbors lie at distance $1$, hence $\alpha_k^{\mathrm{mut}}(d_n)=1.$ However, the minimal feasible slack $\alpha^\star(d_n)$ grows strictly larger, as shown in Table~\ref{tab:expander_gap}. Thus, even in the absence of local scale imbalance, a nontrivial global distortion is required. This empirically illustrates the structural local global gap predicted by Proposition~\ref{prop:intrinsic-gap}.

\begin{table}[t]
\centering
\begin{tabular}{c|c|c}
$n$ & $\alpha_k^{\mathrm{mut}}(d_n)$ & $\alpha^\star(d_n)$ \\
\hline
32  & 1.0 & 5.21 \\
64  & 1.0 & 5.21 \\
128 & 1.0 & 6.00 \\
\end{tabular}
\caption{Shortest path metrics of random $3$ regular graphs. 
Although the local mutual heterogeneity ratio equals $1$, 
the minimal feasible global slack $\alpha^\star$ is strictly larger, demonstrating a strict local–global separation.}
\label{tab:expander_gap}
\end{table}


\paragraph{Scaling of Minimal Slack}

We next study how the minimal slack $\alpha^\star$ scales with the instance size $n$ for shortest path metrics of random $3$-regular graphs. Figure~\ref{fig:scaling1} shows the growth of $\alpha^\star$ as $n$ increases. The minimal slack grows with $n$, indicating that larger instances amplify intrinsic global obstruction.
This behavior is qualitatively consistent with classical lower bounds for tree embeddings of such graph metrics.
It reinforces the view that $\alpha^\star$ captures an intrinsic global geometric obstruction rather than a finite sample artifact.

\begin{figure}[t]
\centering
\begin{tikzpicture}
\begin{axis}[
    width=0.8\linewidth,
    height=6cm,
    xlabel={$n$},
    ylabel={Minimal slack $\alpha^\star(d,k)$},
    xmode=log,
    log basis x={2},
    xtick={16,32,64,128},
    ymin=0,
    grid=major,
    thick
]
\addplot[
    color=black,
    mark=*,
    smooth
] coordinates {
(16,6.49)
(32,12.97)
(64,14.96)
(128,26.44)
};
\end{axis}
\end{tikzpicture}
\caption{
Growth of minimal distortion parameter $\alpha^\star$
with instance size $n$.
}
\label{fig:scaling1}
\end{figure}
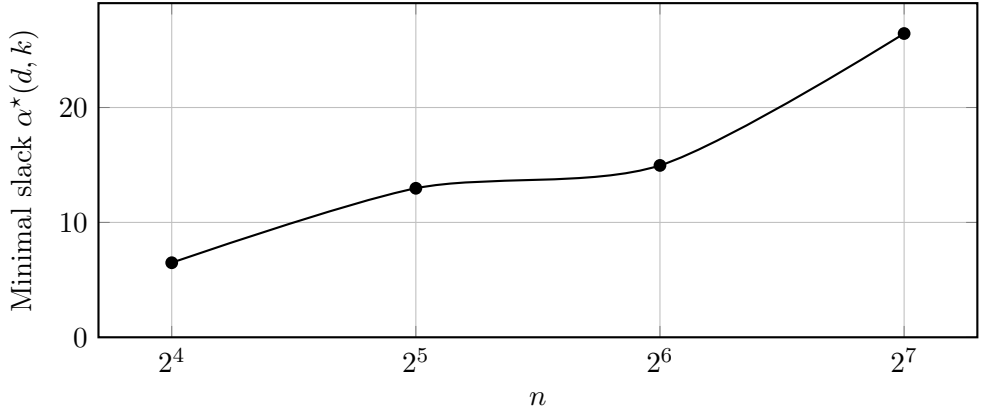
Having established phase behavior and intrinsic global obstruction in synthetic datasets, we now examine whether similar structural phenomena arise in real world data.
\subsection{Real world results}
We empirically study the geometric behavior of the minimal multiplicative slack $\alpha^\star(d,k)$
and compare it against the local slack parameter $\alpha_{\mathrm{mut}}(k)$.
Our experiments span three real world datasets Adult (UCI Census Income) \cite{adult_2}, Statlog (German Credit Data) \cite{german}, Iris \cite{iris_53}.

\paragraph{Obstruction Saturation in Neighborhood Size}

Figure~\ref{fig:alpha_k_small} plots $\alpha^\star(d,k)$ for small subsamples.
Across both Adult (n=100) and German (n=100), we observe a consistent
two phase behavior: $\alpha^\star(d,k)$ increases from $k=1$ to $k=3$ and then stabilizes.
In contrast, $\alpha_{\mathrm{mut}}(k)$ grows with $k$.

On Adult (n=100), $\alpha^\star(d,k)$ increases from $2.53$ at $k=1$
to $2.91$ at $k=3$ and remains constant thereafter.
On German (n=100), $\alpha^\star(d,k) \approx 1.48$ for all tested $k$. This indicates that global ultrametric obstruction is determined by
small scale geometric configurations and saturates at low neighborhood sizes,
whereas local fairness constraints continue to accumulate combinatorial tension. For Iris, the triangle based obstruction yields $\alpha^\star = 2$,
consistent with generic Euclidean geometry and significantly smaller
than the distortion observed in heterogeneous tabular data.

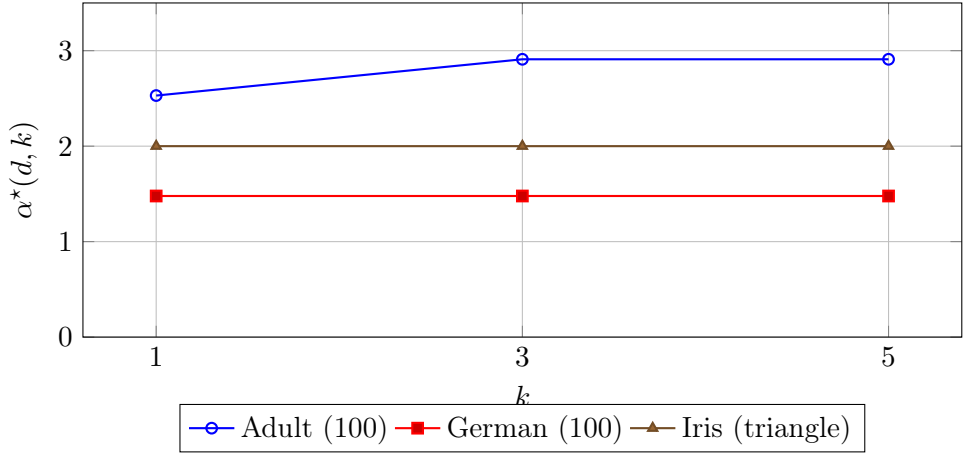
\begin{figure}[t]
\centering
\begin{tikzpicture}
\begin{axis}[
    width=0.8\linewidth,
    height=6cm,
    xlabel={$k$},
    ylabel={$\alpha^\star(d,k)$},
    xtick={1,3,5},
    legend style={at={(0.5,-0.2)},anchor=north,legend columns=3},
    ymin=0,
    ymax=3.5,
    grid=major
]

\addplot+[mark=o, thick] coordinates {
(1,2.53)
(3,2.91)
(5,2.91)
};

\addplot+[mark=square*, thick] coordinates {
(1,1.478)
(3,1.478)
(5,1.478)
};

\addplot+[mark=triangle*, thick] coordinates {
(1,2.0)
(3,2.0)
(5,2.0)
};

\legend{Adult (100), German (100), Iris (triangle)}

\end{axis}
\end{tikzpicture}
\caption{Minimal slack $\alpha^\star(d,k)$ versus neighborhood size $k$.
Obstruction saturates at small $k$ across datasets.}
\label{fig:alpha_k_small}
\end{figure}

\begin{figure}[t]
\centering
\begin{tikzpicture}
\begin{axis}[
    width=0.8\linewidth,
    height=6cm,
    xlabel={Sample size $n$},
    ylabel={$\alpha^\star(k=3)$},
    legend style={at={(0.5,-0.2)},anchor=north,legend columns=2},
    ymin=0,
    ymax=12,
    grid=major
]

\addplot+[mark=o, thick] coordinates {
(100,2.91)
(500,11.34)
};

\addplot+[mark=square*, thick] coordinates {
(100,1.478)
(1000,1.6826)
};

\legend{Adult, German}

\end{axis}
\end{tikzpicture}
\caption{Scaling of minimal slack $\alpha^\star(k=3)$ with sample size.
Adult exhibits strong growth, while German remains stable.}
\label{fig:scaling}
\end{figure}
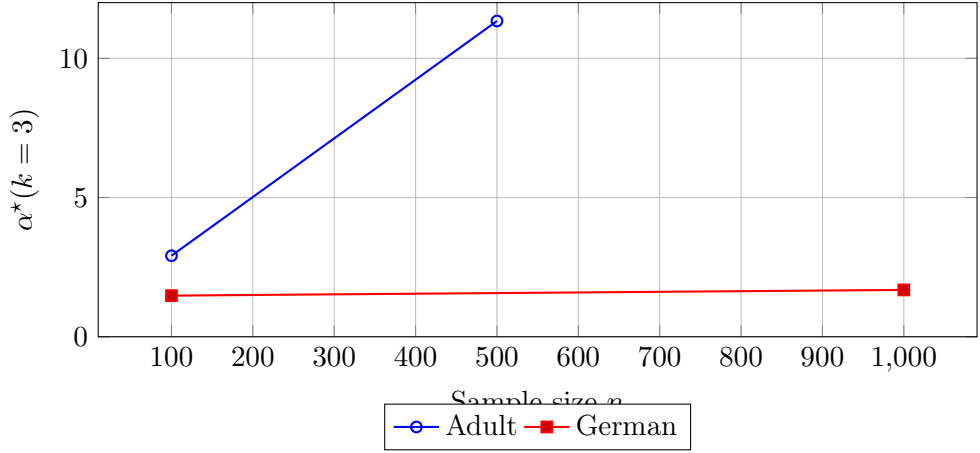

\paragraph{Scaling Behavior}

Figure~\ref{fig:scaling} illustrates how $\alpha^\star(k=3)$
scales with sample size.
For Adult, increasing $n$ from $100$ to $500$
raises $\alpha^\star$ from $2.91$ to $11.34$.
In contrast, German increases only mildly
(from $1.48$ at $n=100$ to $1.68$ at $n=1000$). This suggests that in heterogeneous feature spaces,
rare local geometric violations accumulate with sample size
and significantly amplify global ultrametric distortion.

\paragraph{Algorithmic Baseline: FRT vs.\ FCAC.}
We compare FCAC with the FRT embedding~\cite{fakcharoenphol2004approximating}, 
which produces a dominating HST with expected $O(\log n)$ distortion.
We use the Adult dataset, retain numeric features, standardize them,
and compute the Euclidean metric on $n=800$ sampled points.
Results on \ref{tab:frt_fcac_adult} are averaged over five random seeds.

\begin{table}[h]
\centering
\small
\begin{tabular}{lccc}
\toprule
Method & Mean Slack $\alpha$ & Std & Runtime (s) \\
\midrule
FCAC (MST-based) & 1.00 & 0.00 & 73.19 \\
FRT              & 67.26 & 20.59 & 20.47 \\
\bottomrule
\end{tabular}
\caption{Slack and runtime comparison on Adult ($n=800$).
While FRT provides an expected $O(\log n)$ distortion guarantee
($\log n \approx 6.68$), its empirical slack is substantially larger.
FCAC achieves unit slack at increased computational cost.}
\label{tab:frt_fcac_adult}
\end{table}

Despite the worst-case $O(\log n)$ guarantee, FRT exhibits
substantial empirical distortion on this dataset.
In contrast, the geometry-aware construction attains unit slack,
indicating that the minimal dominating ultrametric closely matches
the intrinsic data geometry.

\section{Conclusion and Future Work}

We study individual fairness in hierarchical clustering through dominated ultrametric embeddings. We identify a sharp local slack threshold, prove stability under bounded perturbations, and establish an intrinsic $\Theta(\log n)$ separation between local feasibility and global realizability. Experiments reveal sharp feasibility transitions and dataset-dependent distortion regimes.

Open directions include characterizing the computational complexity of feasibility, tightening bounds for structured metric classes, developing scalable approximation algorithms, and extending this geometric framework to other multi scale representations and fairness notions.
\newpage
\bibliographystyle{plain}
\bibliography{uai2026-template/uai2026-template/arxiv}

\begin{thebibliography}{10}

\bibitem{abraham2012using}
Ittai Abraham and Ofer Neiman.
\newblock Using petal-decompositions to build a low stretch spanning tree.
\newblock In {\em Proceedings of the forty-fourth annual ACM symposium on Theory of computing}, pages 395--406, 2012.

\bibitem{amagata}
Daichi Amagata.
\newblock Fair k-center clustering with outliers.
\newblock In Sanjoy Dasgupta, Stephan Mandt, and Yingzhen Li, editors, {\em Proceedings of The 27th International Conference on Artificial Intelligence and Statistics}, volume 238 of {\em Proceedings of Machine Learning Research}, pages 10--18. PMLR, 02--04 May 2024.

\bibitem{bartal1996probabilistic}
Yair Bartal.
\newblock Probabilistic approximation of metric spaces and its algorithmic applications.
\newblock In {\em Proceedings of 37th Conference on Foundations of Computer Science}, pages 184--193. IEEE, 1996.

\bibitem{bateni2024scalablealgorithmindividuallyfair}
MohammadHossein Bateni, Vincent Cohen-Addad, Alessandro Epasto, and Silvio Lattanzi.
\newblock A scalable algorithm for individually fair k-means clustering, 2024.

\bibitem{adult_2}
Barry Becker and Ronny Kohavi.
\newblock {Adult}.
\newblock UCI Machine Learning Repository, 1996.
\newblock {DOI}: https://doi.org/10.24432/C5XW20.

\bibitem{bilenko2004integrating}
Mikhail Bilenko, Sugato Basu, and Raymond~J Mooney.
\newblock Integrating constraints and metric learning in semi-supervised clustering.
\newblock In {\em Proceedings of the twenty-first international conference on Machine learning}, page~11, 2004.

\bibitem{cohen2019hierarchical}
Vincent Cohen-Addad, Varun Kanade, Frederik Mallmann-Trenn, and Claire Mathieu.
\newblock Hierarchical clustering: Objective functions and algorithms.
\newblock {\em Journal of the ACM (JACM)}, 66(4):1--42, 2019.

\bibitem{dasgupta2016cost}
Sanjoy Dasgupta.
\newblock A cost function for similarity-based hierarchical clustering.
\newblock In {\em Proceedings of the forty-eighth annual ACM symposium on Theory of Computing}, pages 118--127, 2016.

\bibitem{davidson2005agglomerative}
Ian Davidson and SS~Ravi.
\newblock Agglomerative hierarchical clustering with constraints: Theoretical and empirical results.
\newblock In {\em European conference on principles of data mining and knowledge discovery}, pages 59--70. Springer, 2005.

\bibitem{dwork2011fairnessawareness}
Cynthia Dwork, Moritz Hardt, Toniann Pitassi, Omer Reingold, and Rich Zemel.
\newblock Fairness through awareness, 2011.

\bibitem{fakcharoenphol2004approximating}
Jittat Fakcharoenphol, Satish Rao, and Kunal Talwar.
\newblock Approximating metrics by tree metrics.
\newblock {\em ACM SIGACT News}, 35(2):60--70, 2004.

\bibitem{iris_53}
R.~A. Fisher.
\newblock {Iris}.
\newblock UCI Machine Learning Repository, 1936.
\newblock {DOI}: https://doi.org/10.24432/C56C76.

\bibitem{gan2020data}
Guojun Gan, Chaoqun Ma, and Jianhong Wu.
\newblock {\em Data clustering: theory, algorithms, and applications}.
\newblock SIAM, 2020.

\bibitem{han2023approx}
Lu~Han, Dachuan Xu, Yicheng Xu, and Ping Yang.
\newblock Approximation algorithms for the individually fair k-center with outliers.
\newblock {\em J. of Global Optimization}, 2022.

\bibitem{german}
Hans Hofmann.
\newblock {Statlog (German Credit Data)}.
\newblock UCI Machine Learning Repository, 1994.
\newblock {DOI}: https://doi.org/10.24432/C5NC77.

\bibitem{johnson1967hierarchical}
Stephen~C Johnson.
\newblock Hierarchical clustering schemes.
\newblock {\em Psychometrika}, 32(3):241--254, 1967.

\bibitem{kannan1995tree}
Sampath~K Kannan and Tandy~J Warnow.
\newblock Tree reconstruction from partial orders.
\newblock {\em SIAM Journal on Computing}, 24(3):511--519, 1995.

\bibitem{linial1995geometry}
Nathan Linial, Eran London, and Yuri Rabinovich.
\newblock The geometry of graphs and some of its algorithmic applications.
\newblock {\em Combinatorica}, 15(2):215--245, 1995.

\bibitem{mahabadi2020individual}
Sepideh Mahabadi and Ali Vakilian.
\newblock Individual fairness for k-clustering.
\newblock In {\em International conference on machine learning}, pages 6586--6596. PMLR, 2020.

\bibitem{binita}
Binita Maity, Shrutimoy Das, and Anirban Dasgupta.
\newblock Linear programming based approximation to individually fair k-clustering with outliers, 2024.

\bibitem{maity2025localsearchbasedindividuallyfair}
Binita Maity, Shrutimoy Das, and Anirban Dasgupta.
\newblock Local search-based individually fair clustering with outliers, 2025.

\bibitem{moseley2023approximation}
Benjamin Moseley and Joshua~R Wang.
\newblock Approximation bounds for hierarchical clustering: Average linkage, bisecting k-means, and local search.
\newblock {\em Journal of Machine Learning Research}, 24(1):1--36, 2023.

\bibitem{murtagh1983survey}
F.~Murtagh.
\newblock A survey of recent advances in hierarchical clustering algorithms.
\newblock {\em The Computer Journal}, 26(4):354--359, 11 1983.

\bibitem{murtagh2017algorithms}
Fionn Murtagh and Pedro Contreras.
\newblock Algorithms for hierarchical clustering: an overview, ii.
\newblock {\em Wiley Interdisciplinary Reviews: Data Mining and Knowledge Discovery}, 7(6):e1219, 2017.

\bibitem{papadimitriou1998combinatorial}
Christos~H Papadimitriou and Kenneth Steiglitz.
\newblock {\em Combinatorial optimization: algorithms and complexity}.
\newblock Courier Corporation, 1998.

\bibitem{serre2002trees}
Jean-Pierre Serre.
\newblock {\em Trees}.
\newblock Springer Science \& Business Media, 2002.

\end{thebibliography}

\newpage
\appendix

\section{Constructive Upper Bound}
\label{appendix:a}

\begin{proof}[Proof of Theorem~\ref{thm:sufficiency}]

Let $(V,d)$ be an $n$-point metric space. By the probabilistic tree embedding theorem of
r~\cite{fakcharoenphol2004approximating},
there exists a tree metric $u$ on $V$ such that

$d(i,j) \le u(i,j)
\quad \forall i,j,$ and $u(i,j) \le C \log n \cdot d(i,j)
\quad \forall i,j,$ for some universal constant $C$.

Moreover, the embedding can be chosen so that $u$
is an ultrametric (via hierarchical decomposition).

The embedding guarantees $d(i,j) \le u(i,j),$ that the domination constraint holds.

For any required neighbor pair $(i,j), u(i,j) \le C \log n \cdot d(i,j).$ Thus the multiplicative fairness constraint holds with $\alpha = C \log n.$

Therefore, $d(i,j) \le u(i,j) \le C \log n \cdot d(i,j)$ for all required pairs. Hence $u \in \mathcal U_{IF}^+$ and feasibility follows.

\end{proof}
\section{Proof of Proposition~\ref{prop:intrinsic-gap}}
\label{appendix:c}

We provide a formal proof of the intrinsic gap between
local obstruction and global feasibility.

\begin{proof}
Let $G_n$ be a constant-degree expander graph on $n$ vertices,
with degree $D$ independent of $n$.
Let $d_n$ denote the shortest-path metric on $G_n$.

Fix $k=D$,for every vertex $v$, the $k$ nearest neighbors under $d_n$ are exactly its graph neighbors, each at distance $1$. Since $G_n$ is $D$-regular, neighborhood relations are symmetric:
if $u$ is adjacent to $v$, then $v$ is also adjacent to $u$.
Thus every mutual $k$-nearest-neighbor pair satisfies $d_n(u,v)=1.$

For any vertex $v$ and any two mutual neighbors
$u,w \in N_k(v)$, $\frac{d_n(v,w)}{d_n(v,u)} = \frac{1}{1} = 1.$ Taking the maximum over all such triples yields  $\alpha_k^{\mathrm{mut}}(d_n)=1.$

It is a classical result that every embedding of the
shortest-path metric of a constant-degree expander
into a tree metric incurs distortion $\Omega(\log n)$(\cite{bartal1996probabilistic,fakcharoenphol2004approximating}).

Formally, for any tree metric $u$ satisfying $d_n(x,y) \le u(x,y) \quad \forall x,y,$ there exists $x,y$ such that
$u(x,y) \ge c \log n \, d_n(x,y)$ for a universal constant $c>0$.

Since every ultrametric is a tree metric,
the same lower bound applies to dominating ultrametrics.

Therefore,
\[
\alpha^\star(d,k)
=
\min_{u \in U,\, d_n \le u}
\max_{x,y}
\frac{u(x,y)}{d_n(x,y)}
=
\Theta(\log n).
\]

This establishes the intrinsic gap.
\end{proof}
\paragraph{Edge-level strengthening.}
Classical expander lower bounds imply that the
average stretch over edges is $\Omega(\log n)$.
Hence at least one edge $(x,y)$ satisfies
\[
\frac{u(x,y)}{d_n(x,y)} \ge c \log n.
\]
For $k$ equal to the degree, every edge is a mutual
$k$-nearest neighbor pair.
Thus at least one fairness constrained pair
requires $\Omega(\log n)$ multiplicative slack.

\section{Experimental Results}
In this section, we have added additional experiments.

\subsection{Synthetic dataset}
\paragraph{Fairness Distortion Trade-off} We next study the effect of the neighborhood size $k$ on Gaussian data.
Larger values of $k$ impose fairness constraints on more pairs,
thereby tightening the local Lipschitz requirements.

Figure~\ref{fig:tradeoff} shows that the minimal slack $\alpha^\star$
increases rapidly as $k$ grows.
This reflects the growing geometric tension between enforcing
fairness over larger neighborhoods and maintaining an ultrametric structure.

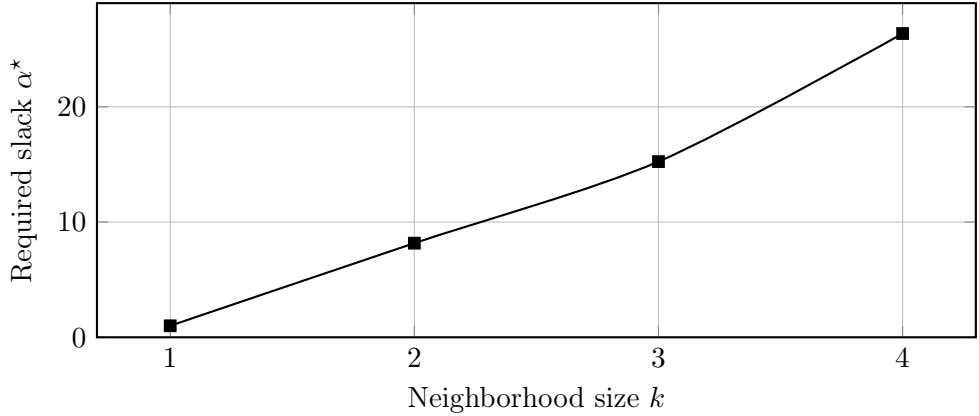
\begin{figure}[t]
\centering
\begin{tikzpicture}
\begin{axis}[
    width=0.8\linewidth,
    height=6cm,
    xlabel={Neighborhood size $k$},
    ylabel={Required slack $\alpha^\star$},
    ymin=0,
    xtick={1,2,3,4},
    grid=major,
    thick
]
\addplot[
    color=black,
    mark=square*,
    smooth
] coordinates {
(1,1.0000)
(2,8.1683)
(3,15.2469)
(4,26.3592)
};
\end{axis}
\end{tikzpicture}
\caption{
Minimal slack $\alpha^\star$ as a function of neighborhood size $k$
on Gaussian data. Increasing $k$ strengthens fairness constraints
and requires progressively larger global distortion.
}
\label{fig:tradeoff}
\end{figure}

The monotone growth of $\alpha^\star$ quantifies the cost of expanding
local fairness neighborhoods, illustrating the trade-off between
stronger local constraints and global hierarchical consistency.

\paragraph{Stability Under Metric Perturbations} Finally, we evaluate robustness on Gaussian data.
Let $\delta$ denote the minimum gap between consecutive
$k$-nearest-neighbor distances.
We perturb the metric so that $\|d-d'\|_\infty \le \epsilon$
and recompute $\alpha^\star$.

For an instance with $\alpha^\star = 17.46$, the observed variation
$|\Delta \alpha| := |\alpha^\star(d')-\alpha^\star(d)|$
is shown in Table~\ref{tab:stability}.

\begin{table}[t]
\centering
\begin{tabular}{c|c}
$\epsilon$ & $|\Delta \alpha|$ \\
\hline
0.0000 & 0.0000 \\
0.0011 & 0.0000 \\
0.0022 & 0.0000 \\
0.0033 & 0.0000 \\
0.0044 & 0.4987 \\
0.0056 & 0.9975 \\
0.0100 & 5.4862 \\
\end{tabular}
\caption{Stability of $\alpha^\star$ under $\ell_\infty$ metric perturbations.
The distortion parameter remains unchanged for sufficiently small $\epsilon$
and changes only once neighborhood identities are altered.}
\label{tab:stability}
\end{table}

For small perturbations ($\epsilon \le 0.003$),
the distortion parameter remains unchanged.
Deviations occur only after nearest-neighbor identities shift,
consistent with the stability guarantee of Theorem~\ref{thm:stability}.

Overall, the experiments support the interpretation of
$\alpha^\star(d)$ as a geometric distortion parameter.
It exhibits sharp feasibility thresholds, reveals intrinsic
local global gap, grows with structural complexity,
increases with stronger fairness constraints,
and remains stable under bounded perturbations.

\subsection{real world dataset}

We empirically study the geometric behavior of the minimal multiplicative slack $\alpha^\star(d,k)$ for Adult \cite{adult_2} and German datasets  \cite{german}, we additionally vary sample size to study scaling effects. Categorical attributes are one hot encoded and all features are standardized. Distances are computed using the Euclidean metric.

The magnitude of $\alpha^\star$ varies substantially across datasets.

\begin{center}
\begin{tabular}{lcc}
\toprule
Dataset & Sample Size & $\alpha^\star(k=3)$ \\
\midrule
German & 100 & 1.48 \\
German & 1000 & 1.68 \\
Adult & 100 & 2.91 \\
Adult & 500 & 11.34 \\
\bottomrule
\end{tabular}
\end{center}

German Credit consistently exhibits low distortion
($\alpha^\star \approx 1.5$--$1.7$),
indicating proximity to hierarchical structure.
In contrast, Adult requires substantially larger slack,
particularly as sample size increases.
Iris exhibits $\alpha^\star = 2$ under triangle based obstruction, consistent with generic Euclidean geometry. These results demonstrate that fairness constrained
hierarchical compatibility is highly dataset dependent.

\paragraph{Local Global Gap}

To highlight the interaction between local and global constraints,
Figure~\ref{fig:gap_k} plots the gap
$\alpha^\star(d,k) - \alpha_{\mathrm{mut}}(k).$
For both Adult and German, the gap is positive at $k=1$,
indicating that global geometry dominates infeasibility.
For larger $k$, the gap becomes negative,
showing that local fairness constraints rapidly overtake
global ultrametric obstruction. This sign reversal marks a structural transition from geometry dominated to combinatorially dominated infeasibility.

Figures comparing $\alpha^\star(d,k)$ to $\alpha_k^{mut}(d)$
should be interpreted with care: experiments enforce one-sided
neighborhood fairness, whereas $\alpha_k^{mut}$ corresponds to
the mutual formulation. The two local thresholds need not coincide.

\begin{figure}[t]
\centering
\begin{tikzpicture}
\begin{axis}[
    width=0.8\linewidth,
    height=6cm,
    xlabel={$k$},
    ylabel={$\alpha^\star(d,k) - \alpha_{\mathrm{mut}}(k)$},
    xtick={1,3,5},
    legend style={at={(0.5,-0.2)},anchor=north,legend columns=2},
    ymin=-12,
    ymax=2,
    grid=major
]

\addplot+[mark=o, thick] coordinates {
(1,1.53125)
(3,-7.74035)
(5,-9.89025)
};

\addplot+[mark=square*, thick] coordinates {
(1,0.47848)
(3,-0.43634)
(5,-0.71542)
};

\legend{Adult (100), German (100)}

\end{axis}
\end{tikzpicture}
\caption{Local global gap $\alpha^\star(d,k) - \alpha_{\mathrm{mut}}(k)$.
The sign reversal reflects transition from global obstruction to local constraint explosion.}
\label{fig:gap_k}
\end{figure}
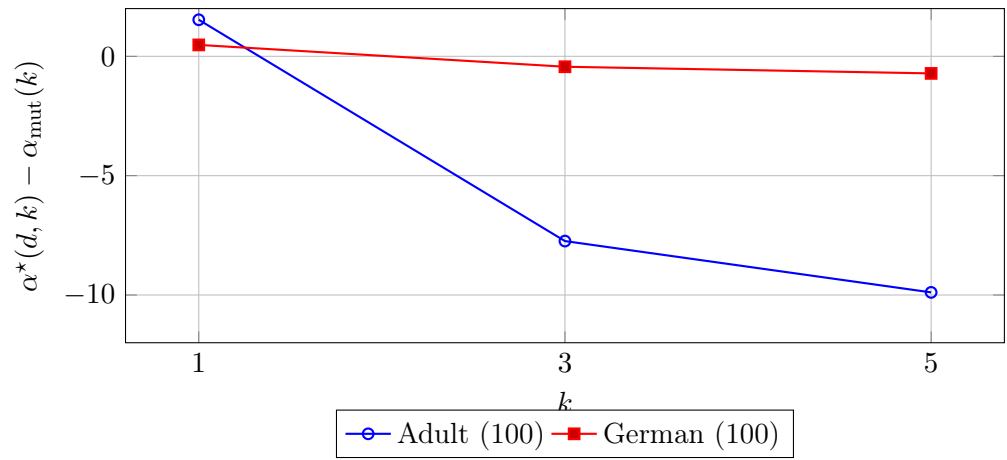

\end{document}